\documentclass[a4paper,12pt,reqno]{amsart} 
        \usepackage{graphicx}
        \graphicspath{ {images/} }
        \usepackage{amssymb}
        \usepackage{graphicx}
        \usepackage{tikz-cd}
        \usepackage{tikz}
        \usepackage{amsmath}
        \usepackage{dsfont}
        \usepackage{shuffle}
        \usepackage{afterpage}
        \usepackage{setspace}
        \graphicspath{ {/Users/emanuele/Documents} }
        \usepackage[utf8]{inputenc}
        
        \usepackage{mathrsfs}
        \usepackage{mathbbol}
\usepackage[margin=1in]{geometry}
\usepackage{amsaddr}
\usepackage[bookmarks=true,
bookmarksnumbered=true, breaklinks=true,
pdfstartview=FitH, hyperfigures=false,
plainpages=false, naturalnames=true,
colorlinks=true,pagebackref=true,
pdfpagelabels]{hyperref}
        
        \usepackage[all]{xy}
        \usetikzlibrary{knots} 
        \usepackage[lite]{amsrefs}
        \usepackage{bbm}

        \newtheorem{thm}{Theorem}[section]
        
        \newtheorem{lem}[thm]{Lemma}

        \newcommand{\eq}[1][r]
           {\ar@<-3pt>@{-}[#1]
            \ar@<-1pt>@{}[#1]|<{}="gauche"
            \ar@<+0pt>@{}[#1]|-{}="milieu"
            \ar@<+1pt>@{}[#1]|>{}="droite"
            \ar@/^2pt/@{-}"gauche";"milieu"
            \ar@/_2pt/@{-}"milieu";"droite"}

          \entrymodifiers={!!<0pt,0.7ex>+} 
        
        \newcommand{\bigon}[4][r]{
            \ar@/^1pc/[#1]^{#2}_*=<0.3pt>{}="HAUT"
            \ar@/_1pc/[#1]_{#3}^*=<0.3pt>{}="BAS"
            \ar@{=>} "HAUT";"BAS" ^{#4}
          }
        
        \newcommand{\bigons}[6][r]{  
            \ar@/^2pc/[#1]^{#2}_*=<0.3pt>{}="HAUT"
            \ar@{}    [#1]     ^*=<0.3pt>{}="MILIEUHAUT"
                               _*=<0.3pt>{}="MILIEUBAS"
            \ar[#1]_(0.3){#3}                  
            \ar@/_2pc/[#1]_{#4}^*=<0.3pt>{}="BAS"
            \ar@{=>} "HAUT";"MILIEUHAUT" ^{#5}
            \ar@{=>} "MILIEUBAS";"BAS" ^{#6}
          }

        \theoremstyle{definition}
        \newtheorem{df}[thm]{Definition}

        \theoremstyle{remark}
        \newtheorem{rmk}[thm]{Remark}

        \allowdisplaybreaks

\title[Universal Approximation]{Universal Approximation of Operators with Transformers and Neural Integral Operators}

\author{Emanuele Zappala$^*$}\thanks{$^*$Corresponding author}
\address{Department of Mathematics and Statistics, Idaho State University
	Physical Science Complex|  921 S. 8th Ave., Stop 8085 | Pocatello, ID 83209, USA} 
\email{emanuelezappala@isu.edu}

\author{Maryam Bagherian}
\address{Department of Mathematics and Statistics, Idaho State University
	Physical Science Complex|  921 S. 8th Ave., Stop 8085 | Pocatello, ID 83209, USA} 
\email{maryambagherian@isu.edu}

\begin{document}

\begin{abstract}
    We study the universal approximation properties of transformers and neural integral operators for operators in Banach spaces.
    In particular, we show that the transformer architecture is a universal approximator of integral operators between H\"older spaces. Moreover, we show that a generalized version of neural integral operators, based on the Gavurin integral, are universal approximators of arbitrary operators between Banach spaces. Lastly, we show that a modified version of transformer, which uses Leray-Schauder mappings, is a universal approximator of operators between arbitrary Banach spaces. 
\end{abstract}

\maketitle

\section{Introduction}

Operator learning constitutes a branch of deep learning dedicated to approximating nonlinear operators defined between Banach spaces. The appeal of this field stems from its capacity to model complex phenomena, such as dynamical systems, for which underlying governing equations remain elusive. Operator learning can be thought to have originated with the theoretical groundwork established in \cite{chen1995universal} and subsequently implemented in \cite{lu2021learning}.

Operator learning has been successfully applied to several fields of science, including learning neural integro-differential equations \cite{NIDE}, integral operators \cites{ANIE_NAT,ANIE}, quantum state tomography \cite{torlai2018neural}, quantum circuit learning \cite{mitarai2018quantum}, control theory \cite{Hwang_Lee_Shin_Hwang_2022}, partial differential equations \cite{PINN_NOperator}, and holographic quantum chromodynamics \cite{hashimoto2021neural} to mention but a few. 

Transformers, \cite{vaswani2017attention}, have emerged as a powerful tool in natural language processing, and they have more recently been employed for addressing problems in operator learning, including dynamical systems. Previous research has explored their use in integral equations \cite{ANIE_NAT,ANIE}, inverse problems \cite{guo2023transformer}, fluid flows \cite{geneva2022transformers}, PDEs \cite{cao2021choose} among others.

Recent theoretical work highlights the strong approximation capabilities of transformer architectures across a range of settings. In \cite{fang2022attention}, the authors showed that a variation of transformers attention mechanism, which includes free parameters, allows recovering of polynomials with zero error over a compact subset of a Euclidean space. This procedure has the drawback of potential exponential scaling of free parameters. It was shown in \cite{yun2019transformers}, that transformers can approximate sequence-to-sequence functions with respect to a suitable ``average'' error. The case of approximating sequence-to-sequence functions with some smoothness, and with infinite dimensional inputs was studied in \cite{takakura2023approximation}, where learnable positional encoding  was employed. More recently, in \cite{jiao2025approximation}, these works have been further extended to show that sequence-to-sequence functions with H\"older and Sobolev regularity can be approximated in the $L^p$ norm.

Recent studies have explored the application of transformers to operator modeling. For instance, \cite{shih2025transformers} demonstrated the efficacy of attention mechanisms in learning operators with low regularity, specializing in Izhikevich and tempered fractional LIF neuron models. This work builds upon previous research \cite{lanthaler2022error,deng2022approximation} that established near-optimal error bounds for learning nonlinear operators between infinite-dimensional Banach spaces, notably DeepONets \cite{lu2019deeponet}. In addition, the work of \cite{shih2025transformers} shows that on Banach spaces with the approximation property, transformers with positional encoding are universal approximators, with the addition of an encoder and a decoder, by applying the results of \cite{cordonnier2019relationship} to uniformly approximate continuous functions via CNNs. 

The case of transformers without positional encoders, and without additional encoders and decoders, have not currently been extensively explored in the case of operator learning problems. In fact, the results of \cite{yun2019transformers,takakura2023approximation,jiao2025approximation} when applied directly to the proof of \cite{shih2025transformers} in this setting, give error bounds that are either not uniform, or are restricted to some classes of functions that are not directly related to the construction in \cite{shih2025transformers}.  

The main result of the present paper is to establish the transformer architecture without positional encoding as a universal approximator for integral operators between Hölder spaces.  Our methods can also be applied for any equivariant universal approximator of sequence-to-sequence mappings, of which the transformer is an example. Along the way, we show that adding Leray-Schauder mappings to transformers, and relaxing the assumption on equivariance, one obtains a universal approximator of operators between arbitrary Banach spaces. Furthermore, we show that a generalized form of neural integral operators can approximate arbitrary operators between Banach spaces.

The paper is organized as follows. Section~\ref{sec:pre} outlines the necessary preliminaries. In Section~\ref{sec:transformer}, we prove that transformers are universal approximators for Hölder spaces. Section~\ref{sec:NIE} investigates the universal approximation capabilities of neural integral operators for operators in Banach spaces. We conclude with Section~\ref{sec:fp}, exploring further perspectives and open questions.

\section{Preliminaries}\label{sec:pre}
\subsection{Notation}
Throughout the article, we assume a discretization of the interval $[a, b]$ as $a = t_0 < t_1 < \cdots < t_{n-1} < t_n = b$. Unless otherwise explicitly stated, over any subinterval, i.e. $[a, t_1], \cdots, [t_{n-1},b]$, the point $x_i\in [t_i, t_{i+1}]$.

Let $X= C^{k,\alpha}([a,b])$ denote the H\"older space of $k$-differentiable functions with $k^{\rm th}$ derivative which are H\"older continuous. We let $T: X\longrightarrow X$ be an integral operator. Consider the values $x_i$, and $z_k = y(t_k)$ for $i, k=1,\ldots , n$, where $x_i$ and $t_k$ vary among the points of the discretization determined by the integration approximation.

Let the kernel of $T$ be the function $G: \mathbb R\times [a,b]\times [a,b] \longrightarrow \mathbb R$ which is continuous with respect to all variables, and is continuously differentiable with respect to the first and third variables. Here, $G$ takes as input a function $u$, and two ``time'' variables from the interval $[a,b]$. Then, $T$ is defined according to the assignment
\begin{eqnarray*}
    T(u)(x) = \int_a^bG(u(t),x,t)dt.
\end{eqnarray*}
In other words, $T$ is a Urysohn integral operator. Below, we will consider some additional conditions on the regularity of $G$, for our results to hold.  

For a function $f: \mathbb R^n \longrightarrow \mathbb R^m$, we indicate by $f^i(x_1, \ldots, x_n)$ the $i^{\rm th}$ entry of the vector $f(x_1, \ldots, x_n)$. 

Following the previous conventions, for each $i=1,\ldots , n$ in the rest of the paper we indicate 
\begin{eqnarray*}
    T^i_u(x_1, \ldots, x_n) := \sum_{k=1}^n G(z_k,x_i,t_k)w_k,
\end{eqnarray*}
where the $w_k$ are weights of a quadrature, as will be defined below, and $z_k$ are the values of the input function $u$ on the discretization points. Observe that fixed $u$, each $z_k$ is determined by $(x_1,\ldots, x_n)$, which justifies the notation.  
Given a function $F: \mathbb R^n \longrightarrow \mathbb R^n$, we introduce the following norm
 \begin{eqnarray}\label{def:comp_pnorm}
       \|F\|_p = \left (\sum_{i=1}^n \int_{[a,t_1]\times \cdots \times [t_{n-1},b]}\, |F^i(x_1,\ldots,x_n)|^p\, d\mu\right )^\frac{1}{p}, 
    \end{eqnarray}   
 which corresponds to the Euclidean norm of the component-wise $p$-norm, and can be applied to $T_u^i$ as a particular case. This $p$-norm will be the objective of the approximation results over H\"older spaces for transformers.

\subsection{Leray-Schauder mappings}
Leray-Schauder mappings \cite{precup2002theorems,krasnoselskii} are a fundamental tool in nonlinear functional analysis. They are used to establish the existence of solutions to nonlinear problems by extending the concept of degree from finite-dimensional spaces to infinite-dimensional Banach spaces. Roughly speaking, given a compact $K$ and a finite dimensional subspace of a Banach space, we want to find a continuous mapping onto the finite subspace satisfying the property that the norm of elements in the compact and the image of the mapping is smaller than a fixed $\epsilon >0$.
Following \cite{krasnoselskii}, we define Leray-Schauder mappings on a Banach space $X$ as follows. Let $K$ be a compact, $E_n$ an $n$-dimensional space spanned by the $\epsilon$-net $x_1, \ldots, x_n$ of $K$. We define $P_n : X \longrightarrow E_n$ by the assignment
\begin{eqnarray*}
    P_n x = \frac{\sum_{i=1}^{n} \mu_i(x) x_i}{\sum_{i=1}^n\mu_i(x)},
\end{eqnarray*}
where 
\begin{eqnarray}
\mu_i(x) = \begin{cases}
\epsilon - \|x-x_i\|, \hspace{1cm} \|x-x_i\| \leq \epsilon\\
0, \hspace{2.855cm} \|x-x_i\| > \epsilon
\end{cases}
\end{eqnarray}
for all $i=1, \ldots , n$. One sees directly that this mapping satisfies the property that 
\begin{eqnarray*}
    \|x - P(x)\| < \epsilon,
\end{eqnarray*}
for all $x\in K$. 

\subsection{Gavurin neural integral operator $\mathfrak T(x)$}

We hereby recall the definition of Gavurin integral \cite{Gavurin}, and use this to define a class of neural integral operators of relevance for the rest of the paper. 

We follow \cites{kantorovich2016functional,Gavurin}, and consider a mapping $R: [x_0, \bar x]\subset X \to  B(X,Y)$, where $B(X,Y)$ denotes the space of bounded linear operators from $X$ to $Y$, two Banach spaces. Then, the integral of $R$ is defined in a procedure similar to the classical Riemann integral, as
    \begin{eqnarray}\label{eqn:def_int}
        \int_{x_0}^{x_0+\Delta x} R(x) dx = \int_0^1 R(x_0+t\Delta x) \Delta x \, dt = \lim_{n\to \infty} \sum\limits_{k=0}^{n-1} R(x_0+\tau_k\Delta x)\Delta x\,  (t_{k+1}-t_k),
    \end{eqnarray} 
when the limit exists and it is independent on the choices of $\tau_k$. We indicate the action of the integrap operator on an element $z$ of $X$ as $\int_{x_0}^{x_0+\Delta x} R(x)(z) dx$

We define a {\it Gavurin integral operator} to be an integral operator $T: X \longrightarrow Y$ which is the point-wise convex sum of integral operators as in Equation~\eqref{eqn:def_int}
\begin{eqnarray}\label{eqn:Gavurin_operator}
    T(z) = \sum_i \eta_i(z)[F_i(z) + \int_{x_0}^{x_0+\Delta x} R_i(x)(z) dx],
\end{eqnarray}
where for every $z$, $\sum_i \eta_i(z) = 1$, i.e. the coefficients are a partition of unity, and $F_i$ are arbitrary linear operators. 

We define a {\it Gavurin neural integral operator} to be a Gavurin integral operator where $F_i$ and $R_i$ consist of a Leray-Schauder mapping composed with a neural network. 

\subsection{Models with Permutation-equivariance}
Equivariance can be formalized using the concept of a $G$-set for a group $G$, consisting of a set $S$ and where $G$ acts on $S$ from the right. Given two $G$-sets, $X$ and $Y$, for the same group $G$, then a function $f : X \to Y$ is said to be \textit{equivariant} if \(
f(x \cdot g) = f(x) \cdot g,\)
for all $g \in G$ and all $x \in X$.

A \textit{Transformer} is a sequence-to-sequence model composed of an encoder and a decoder, both built from stacks of self-attention and feedforward layers. A Transformer defines a function
\(
\mathrm{T}: \mathcal{X}^n \rightarrow \mathcal{Y}^m, 
\)
where \( \mathcal{X} \) and \( \mathcal{Y} \) are the input and output token spaces, respectively, and \( n \), \( m \) are the input and output sequence lengths. Given an input sequence of tokens \( x = (x_1, \dots, x_n) \), each token is mapped to a vector via an embedding function
\(
\mathbf{e}_i = \mathrm{Embed}(x_i) \in \mathbb{R}^d. 
\)
A positional encoding \( \mathbf{p}_i \in \mathbb{R}^d \) is added to retain order information
\(
\mathbf{z}_i^{(0)} = \mathbf{e}_i + \mathbf{p}_i.
\)
At each layer \( \ell \), the input \( \mathbf{Z}^{(\ell-1)} = (\mathbf{z}_1^{(\ell-1)}, \dots, \mathbf{z}_n^{(\ell-1)}) \) is transformed via multi-head self-attention:
\[
\mathrm{Attention}(Q, K, V) = \mathrm{softmax}\left( \frac{QK^\top}{\sqrt{d_k}} \right)V
\]
\[
\mathbf{z}_i^{(\ell,\text{attn})} = \mathrm{MultiHead}(\mathbf{z}_i^{(\ell-1)}) = \mathrm{Concat}(h_1, \dots, h_H)W^O
\]
where each head \( h_j \) is
\[
h_j = \mathrm{Attention}(Q_j, K_j, V_j), \quad Q_j = \mathbf{Z}^{(\ell-1)} W_j^Q, \; K_j = \mathbf{Z}^{(\ell-1)} W_j^K, \; V_j = \mathbf{Z}^{(\ell-1)} W_j^V.
\]
Each position \( i \) is updated via a position-wise feedforward network:
\[
\mathbf{z}_i^{(\ell,\text{ff})} = \mathrm{FFN}(\mathbf{z}_i^{(\ell,\text{attn})}) = \sigma(\mathbf{z}_i W_1 + b_1) W_2 + b_2.
\]
The full update rule at each layer \( \ell \) is
\begin{align*}
\mathbf{z}_i^{(\ell,\|\cdot\|_1)} &= \mathrm{LayerNorm}(\mathbf{z}_i^{(\ell-1)} + \mathbf{z}_i^{(\ell,\text{attn})}),\\
\mathbf{z}_i^{(\ell)} &= \mathrm{LayerNorm}(\mathbf{z}_i^{(\ell,\text{norm1})} + \mathbf{z}_i^{(\ell,\text{ff})}).
\end{align*}
In the context of Transformers, particularly in machine learning models, \textit{equivariance} refers to equivariance with respect to the natural symmetric group action on sequences of inputs-outputs. More specifically, for the function \( f: X \to Y \) (e.g., a neural network layer), and the group \( G \) acting on both \( X \) and \( Y \), the function \( f \) is said to be equivariant with respect to \( G \) if
\[
f(x \cdot g) = f(x) \cdot g, \quad \forall g \in G, \; x \in X,
\]
where \( x \cdot g \) and \( f(x) \cdot g \) denote the actions of \( g \) on \( X \) and \( Y \), respectively.

Deep Sets~\cite{zaheer2017deep} are explicitly designed to be permutation invariant or equivariant by applying a shared function to each element followed by a symmetric aggregation operation (e.g., summation). Neural Functional models~\cite{zhou2024universal} generalize this idea by treating input data as functions over domains and ensure equivariance under transformations such as permutations by construction. Transformers~\cite{vaswani2017attention}, before the addition of positional encodings, are naturally permutation equivariant due to the symmetric structure of the self-attention mechanism; i.e., each token attends to all others through shared attention weights, making the output invariant to the order of inputs under shared weights. However, once positional information is introduced, this permutation symmetry is typically broken.

The approach pursued in this article can be applied to any architecture which is equivariant and a universal approximator of sequnce-to-sequence mappings, in the norm defined in \eqref{def:comp_pnorm}. In other words, given an approximator of sequence-to-sequence mappings, we can show that there is a corresponding construction which is a universal approximator of nonlinear operators between the H\"older spaces. We pose the following useful definition. 

In this article, we will use the term {\it generalized transformer} to denote a mapping $\mathcal W$ defined on a space of functions as follows. We have a locally finite partition of unity $\{\eta_i\}$ defined on the space subordinate to some open covering $\mathcal U$, and we let $W_U$ denote a standard transformer for each open $U$ in the covering $\mathcal U$, whether it has positional encoding or not. For a function $f\in U$, $W_U(f)$ is defined by applying $f$ to some pre-defined sampling scheme of the domain of the space of functions, and applying $W_U$ on the corresponding outputs of $f$. We then set 
\begin{eqnarray*}
    \mathcal W(f) &=& \sum_i\eta_i(f)W_U(f). 
\end{eqnarray*}
Each sum is finite, since the partition of unity is locally finite. Details of this procedure are provided below in Section~\ref{sec:transformer}.

\section{Universal Approximation for Transformers}\label{sec:transformer}

In this section, we let $C^{k,\alpha}([a,b])$ denote the H\"older space of $k$-differentiable functions with $k^{\rm th}$ derivative which are H\"older continuous, and let $T: C^{k,\alpha}([a,b]) \longrightarrow C^{k,\alpha}([a,b])$ be a given (possibly nonlinear) integral operator. We want to show that there is a transformer architecture $\mathcal W$ that approximates $T$ with arbitrarily high precision in the following sense. For an arbitrary choice of $\epsilon >0$, there exist a discretization $a = t_0 < t_1 < \cdots < t_{n-1} < t_n = b$ of $[a,b]$ and a transformer architecture $\mathcal W$ such that for each $u\in C^{k,\alpha}([a,b])$, we have that 
\begin{eqnarray}
    \left (\sum_{i=1}^n \int_{[a,t_1]\times \cdots \times [t_{n-1},b]}|T^i_u(x_1,\ldots,x_n)-\mathcal W^i_u(x_1,\ldots,x_n)|^pd\mu\right )^\frac{1}{p} < \epsilon, 
\end{eqnarray}
where $d\mu$ is a shorthand notation for $dx_1\cdots dx_n$. 

We will need several lemmas before obtaining the main result. 

We set $X$ to be a H\"older space $C^{k,\alpha}([a,b])$ (with $k\geq 1$), of functions defined on some closed interval $[a,b]$. Let $T: X \longrightarrow X$ denote an integral operator, where we assume that $T$ is a Urysohn (possibly nonlinear) integral operator, and that the kernel of $T$ is a function $G: \mathbb R\times [a,b]\times [a,b] \longrightarrow \mathbb R$ which is continuous with respect to all variables, and is continuously differentiable with respect to the first and third variables. We let $\mathbb B_\rho$ be a ball centered at the origin with radius $\rho$ fixed. In this article, we will consider only numerical integration schemes that are non-adaptive, and that are given by continuous formulas with respect to the sampled points. 

\begin{lem}\label{lem:int_scheme}
        Let $T$ and $X$ be as above. Then, we can find an integration scheme, $\mathcal T$ such that 
        \begin{eqnarray}
            \|T(y) - \mathcal T(y)\|_{\infty} < \epsilon,
        \end{eqnarray}    
        for all $y\in \bar {\mathbb B}_\rho$, where $\|\cdot \|_{\infty}$ denotes the uniform norm. 
\end{lem}
\begin{proof}
    We denote by $G : \mathbb R\times [a,b]\times [a,b] \longrightarrow \mathbb R$ the kernel of the integral operator $T$. Therefore,
    \[T(y)(x) = \int_a^b G(y(t),x,t)dt,\] for all $x\in [a,b]$. For $y \in \bar {\mathbb B}_\rho$, from $\|y\| \leq \rho$ it follows that $\sup_{[a,b]} |y(x)| \leq \rho$. Therefore, $G(y(t),x,t)$ takes values $y(t)$ in $[-\rho,\rho]$. Let $M_1$ denote the value 
    \begin{eqnarray*}
        M_1 := \max_{[-\rho,\rho]\times [a,b]\times [a,b]} |\partial_y G(y(t),x,t)|,
    \end{eqnarray*} and let 
    \begin{eqnarray*}
        M_2 := \max_{[-\rho,\rho]\times [a,b]\times [a,b]} |\partial_t G(y(t),x,t)|.
    \end{eqnarray*}
  Let us choose a quadrature rule  with $n$ points $\{q_i\}_{i=1}^n$ selected from $[a,b]$ to approximate the integral operator $T$. We indicate the quadrature by $\mathcal T$: $T(y)(x)  = \int_a^bG(y(t),x,t)dt \approx \sum_{i=1}^n G(y(q_i),x,q_i)w_i =: \mathcal T(y)(x)$, where the coefficients $w_i$ are the weights of the quadrature, and $G(y(q_i),x,q_i)$ is continuous with respect to $x$. Also, $G(y(t),x,t)$ is differentiable with respect to the variable $t$ of integration with derivative given by
  \begin{eqnarray*}
      \frac{dG(y(t),x,t)}{dt} = \partial_t G(y(t),x,t) + \partial_y G(y(t),x,t)\cdot y'(t),
  \end{eqnarray*}
     for any chosen $y\in \bar{\mathbb B}_\rho$. For the sake of simplicity, we use the forward rectangle method in what follows, but a similar procedure (upon suitably changing the error bounds and assuming that the index $k$ of the H\"older space is large enough) can be performed for different integration rules, such as the Midpoint, Trapezoidal, corrected Trapezoidal, Cavalieri-Simpson, Boole etc. It is a known result that the error in performing the numerical integration of $\int_a^bg(t)dt$ is bounded by $\frac{b-a}{2n}\sup_{[a,b]} |g'(t)|$ \cite{young1988survey}. Therefore, for each fixed $x_0\in [a,b]$ we have the bound
    \begin{eqnarray*}
        |T(y)(x_0) - \mathcal T(y)(x_0)| &\leq& \frac{b-a}{2n}\sup_{[a,b]} |\partial_t G(y(t),x_0,t) + \partial_y G(y(t),x_0,t)\cdot y'(t)|\\
                                         &\leq& \frac{b-a}{2n}[ \sup_{[a,b]} |\partial_t G(y(t),x_0,t)| + \sup_{[a,b]}|\partial_y G(y(t),x_0,t)\cdot y'(t)|]\\
                                         &\leq& \frac{b-a}{2n}(M_1\rho + M_2),
    \end{eqnarray*}  
    from which we derive $|T(y)(x) - \mathcal T(y)(x)| \leq \frac{b-a}{2n}(M_1\rho + M_2)$ for all $y\in \bar {\mathbb B}_\rho$ and all $x\in [a,b]$. Upon choosing $n > (b-a)\frac{M_1\rho + M_2}{2\epsilon}$, we find that $\|T(y) - \mathcal T(y)\|_{\infty} < \epsilon$ uniformly in $\bar {\mathbb B}_\rho$.
\end{proof}

A variation of the following argument appeared also in \cite{ANIE,ANIE_NAT}. We provide here a more formal approach. Before stating and proving the result, we pose the following definition. 

\begin{df}\label{def:actions}
    {\rm 
Let $\sigma$ denote a permutation of $n$ elements, i.e. $\sigma\in \Sigma_n$ where $\Sigma_n$ is the permutation group on $n$ elements. Then, $\sigma$ acts on the tuples of type $(z_1, \ldots, z_n, x_1, \ldots , x_n)\in \mathbb R^{2n}$ through the induced diagonal permutation representation, which just reorders the entries $z_1, \ldots , z_n$ and $x_1, \ldots , x_n$ according to the permutation $\sigma$: 
            $$(z_1, \ldots, z_n, x_1, \ldots, x_n)\cdot_1 \sigma := (z_{\sigma(1)}, \ldots, z_{\sigma(n)}, x_{\sigma(1)}, \ldots, x_{\sigma(n)}).$$ 
            We indicate by $\cdot_2$ the permutation action of $\Sigma_n$ on $n$-tuples as
            $$
                    \begin{bmatrix}
                                        x_1\\
                                        \vdots\\
                                        x_n
                                 \end{bmatrix}\cdot_2  \sigma = 
                                 \begin{bmatrix}
                                       x_{\sigma(1)}\\
                                       \vdots\\
                                       x_{\sigma(n)}
                                 \end{bmatrix}.
            $$
            We will denote by $\cdot_2$ also the permutation action on tuples such as $(x_1,\ldots, x_n)$ to shorten notation. 
    }
\end{df}

\begin{lem}\label{lem:equiv_gen}
    The integration scheme from Lemma~\ref{lem:int_scheme} induces a continuous function $F: \mathbb R^n\times [a,b]^n \longrightarrow \mathbb R^n$. Moreover, $F$ is equivariant with respect to the actions defined in Definition~\ref{def:actions}.
\end{lem}
\begin{proof}
        We set $F^i(z_1, \ldots, z_n,x_1,\ldots, x_n) := \sum_{k=1}^n G(z_k,x_i,t_k)w_k$, for each $i=1,\ldots , n$, where the points $t_1, \ldots, t_n$ are determined by the integration scheme defined in Lemma~\ref{lem:int_scheme}. 
        We can therefore define $F$ as the column vector $[F^i]_{i=1}^n$, as a function of the tuple $(z_1, \ldots, z_n, x_1, \ldots, x_n)$. Since each function $F^i$ is continuous with respect to its entries $z_k$ and $x_i$, it follows that the function $F$ is also continuous. 
        
        To show that the function $F$ just defined is permutation equivariant, using the same notations introduced above, we compute
        \begin{eqnarray*}
                F( (z_1, \ldots, z_n, x_1, \ldots, x_n)\cdot_1 \sigma) &=& F(z_{\sigma(1)}, \ldots, z_{\sigma(n)}, x_{\sigma(1)}, \ldots, x_{\sigma(n)})\\
                &=& \begin{bmatrix}
             F^{\sigma(1)}(z_{\sigma(1)}, \ldots, z_{\sigma(n)},x_{1}, \ldots, x_{n})\\
                \vdots\\
            F^{\sigma(n)}(z_{\sigma(1)}, \ldots, z_{\sigma(n)},x_{1}, \ldots, x_{n})
              \end{bmatrix}\\
          &=&  \begin{bmatrix}
                    \sum_{k=1}^n G(z_{\sigma(k)},x_{\sigma(1)},t_{\sigma(k)})w_{\sigma(k)}\\
                    \vdots\\
                    \sum_{k=1}^n G(z_{\sigma(k)},x_{\sigma(n)},t_{\sigma(k)})w_{\sigma(k)}
               \end{bmatrix}\\
          &=&  \begin{bmatrix}
                    \sum_{k=1}^n G(z_k,x_{\sigma(1)},t_k)w_k\\
                    \vdots\\
                    \sum_{k=1}^n G(z_k,x_{\sigma(n)},t_k)w_k
               \end{bmatrix}\\
          &=&  \begin{bmatrix}
                    F^{\sigma(1)}(z_1, \ldots, z_n,x_{1}, \ldots, x_{n})\\
                    \vdots\\
                    F^{\sigma(n)}(z_1, \ldots, z_n,x_{1}, \ldots, x_{n})
              \end{bmatrix}\\
          &=&   F(z_1, \ldots, z_n, x_1, \ldots, x_n) \cdot_2 \sigma. 
        \end{eqnarray*}
        This shows that $F$ is equivariant with respect to the actions. 
\end{proof} 

In fact, we will only need a subcase of the previous lemma. We formulate this in the following result. 

\begin{lem}\label{lem:equiv}
    For any choice of function $y$ in the H\"older space $C^{k,\alpha}([a,b])$, the integration scheme from Lemma~\ref{lem:int_scheme} induces a permutation equivariant continuous function $F_y: \mathbb [a,b]^n \longrightarrow \mathbb R^n$.
\end{lem}
\begin{proof}
    For fixed $y\in C^{k,\alpha}([a,b])$, the function $F$ constructed in Lemma~\ref{lem:equiv_gen} will have the first $n$ entries only depending on $t_i$, since they are obtained by evaluating $y$ at the points $t_i$, for $i= 1, \ldots , n$. The points $t_i$, however, are determined by the integration scheme and are therefore fixed. 
    This induces a function $F_y : \mathbb [a,b]^n \longrightarrow \mathbb R^n$ which is continuous. To show that the map $F_y$ is permutation equivariant with respect to the $\cdot_2$ action, we use the equivariance of $F$ proved in Lemma~\ref{lem:equiv_gen} in the following way. First, observe that $F$ is independent of any permutation with respect to the first $n$ coordinates, since they are all summed together. The diagonal permutation action $\sigma$ on the entries of $F(y_1, \ldots , y_n, x_1,\ldots , x_n)$ is the same as the action of $\sigma$ on the entries of $F_y(x_1, \ldots , x_n)$ for $y_1, \ldots , y_n$ fixed as $y_r = y(t_r)$. Therefore, the equivariance of Lemma~\ref{lem:equiv_gen} completes the proof.
\end{proof}
With notation as above, we have the following result. 
\begin{lem}\label{lem:int_trans}
    Let $y$ be in the ball $\bar{\mathbb B}_\rho$ of radius $\rho$ in the H\"older space $C^{k,\alpha}([a,b])$. Then, for any choice of $\epsilon >0$,  we can find a transformer $W_y: \mathbb R^n \longrightarrow \mathbb R^n$ such that
    \begin{eqnarray}\label{eqn:local_transformer}
        \left (\sum_{i=1}^n \int_{[a,t_1]\times \cdots \times [t_{n-1},b]}|F^i_y(x_1,\ldots,x_n)- W^i_y(x_1,\ldots,x_n)|^pd\mu\right )^\frac{1}{p} < \epsilon.
    \end{eqnarray} 
    Over a compact $\mathbb K \subset C^{k,\alpha}([a,b])$, this approximation can be done uniformly in $y$ with a generalized transformer $\mathcal W$: 
    \begin{eqnarray}\label{eqn:total_transformer}
        \left (\sum_{i=1}^n \int_{[a,t_1]\times \cdots \times [t_{n-1},b]}|F^i_y(x_1,\ldots,x_n)- \mathcal W^i(y)(x_1,\ldots,x_n)|^pd\mu\right )^\frac{1}{p} < \epsilon.
    \end{eqnarray} 
\end{lem}
\begin{proof}
        Applying Lemma~\ref{lem:equiv}, for any choice of $y\in \bar{\mathbb B}_\rho$, $\mathcal T$ induces a continuous and permutation-equivariant map $F_y: \mathbb [a,b]^n \longrightarrow \mathbb R^n$. Applying the results of \cite{yun2019transformers}, we can find a transformer architecture $W_y$, which depends on the given $y$, such that $\|F_y - W_y\|_p < \epsilon$. Then, we have the following
        \begin{eqnarray*}
              \|F_y - W_y\|_p
              &=& \left (\sum_{i=1}^n \int_{[a,b]^n}|F^i_y(x_1,\ldots,x_n)-W_y^i(x_1,\ldots,x_n)|^pd\mu \right )^\frac{1}{p}\\
              &\geq& \left (\sum_{i=1}^n \int_{[a,t_1]\times \cdots \times [t_{n-1},b]}|F^i_y(x_1,\ldots,x_n)-W^i_y(x_1,\ldots,x_n)|^pd\mu\right )^\frac{1}{p}.
        \end{eqnarray*}  
        We therefore obtain
        \begin{eqnarray*}
            \left (\sum_{i=1}^n \int_{[a,t_1]\times \cdots \times [t_{n-1},b]}|F^i_y(x_1,\ldots,x_n)-W^i_y(x_1,\ldots,x_n)|^pd\mu\right )^\frac{1}{p} < \epsilon.
        \end{eqnarray*}

    For the second part of the statement, suppose that $\mathbb K$ is compact. Then, by the boundedness of compacts there exists a ball $\mathbb B_\rho$ of radius $\rho>0$ that contains $\mathbb K$. For any choice of $y \in \mathbb K\subset \bar{\mathbb B}_\rho$, we know from the first part of the lemma that $\mathcal T(y)$ is approximated arbitrarily well by some transformer $W_y$. 

    Observe that $\mathcal T(y)$ depends continuously on $y$ by construction (see Lemma~\ref{lem:int_scheme}). Moreover, as discussed in  Lemma~\ref{lem:int_scheme}, in $\bar{\mathbb B}_\rho$ the ranges of the functions $y$ are contained in $[-\rho,\rho]$. Using the Heine-Cantor Theorem for $G$ over the interval $[-\rho,\rho]$ and the definition of $F_y$ (given in Lemma~\ref{lem:equiv_gen}) induced by $\mathcal T(y)$, we can see that upon choosing an $r>0$ small enough, for any $y,z$ such that $\|y-z\|_{k,\alpha} < r$ we have the inequality
    \begin{eqnarray}\label{ineq:local_p}
         \left (\sum_i\int_{[a,t_1]\times \cdots \times [t_{n-1},b]} |F^i_y(x_1\ldots, x_n) - F^i_z(x_1,\ldots, x_n)|^pd\mu\right )^{\frac{1}{p}} < \frac{\epsilon}{2}.
    \end{eqnarray}
    
    We choose an $r$-net for the compact $\mathbb K$, consisting of finitely many points $y_1, \ldots , y_n$, with corresponding balls $\mathbb B_j$ with radii less than $r$. For each $y_j$, we find a transformer $W_j$ such that   
    \begin{eqnarray*}
        \left (\sum_{i=1}^n \int_{[a,t_1]\times \cdots \times [t_{n-1},b]}|F^i_j(x_1,\ldots,x_n)- W^i_j(x_1,\ldots,x_n)|^pd\mu\right )^\frac{1}{p} < \frac{\epsilon}{2},
    \end{eqnarray*}
    where we have used $j$ in the subscript of $F^i$ instead of $y_j$ for simplicity. 
    There exists a partition of unity $\{\eta_k\}$ subordinate to the covering $\{\mathbb B_k\}$. We use such partition of unity to define a transformer $\mathcal W$ as a combination of the $W_k$ as
    \begin{eqnarray*}
         \mathcal W(y) = \sum_k \eta_k(y) W_k.
    \end{eqnarray*}
    For any $y$ in $\mathbb K$, making use of Jensen's inequality, we have
    \begin{eqnarray*}
           \lefteqn{\sum_{i=1}^n \int_{\Omega}\Big|F^i_y(x_1,\ldots,x_n)- \mathcal W^i(y)(x_1,\ldots,x_n)\Big|^pd\mu}\\
           &=& \sum_{i=1}^n \int_{\Omega} \Big|\sum_k \eta_k(y)F^i_y(x_1,\ldots,x_n) - \sum_k \eta_k(y) W^i_k(x_1,\ldots,x_n)\Big|^pd\mu\\
           &\leq& \sum_{i=1}^n\sum_k\eta_k(y) \int_{\Omega}\Big|F^i_y(x_1,\ldots,x_n)-W^i_k(x_1,\ldots,x_n)\Big|^pd\mu\\
           &\leq& \sum_{i=1}^n\bigg[2^{p-1}\sum_k \eta_k(y) \int_{\Omega}\Big|F^i_y(x_1,\ldots,x_n)-F^i_k(x_1,\ldots,x_n)\Big|^pd\mu\\
           && + 2^{p-1}\sum_k \eta_k(y)\int_{\Omega} \Big|F^i_k(x_1,\ldots,x_n) - W^i_k(x_1,\ldots,x_n)\Big|^pd\mu\bigg]\\
           &=& 2^{p-1}\sum_k\eta_k(y)\sum_{i=1}^n\int_{\Omega}\Big|F^i_y(x_1,\ldots,x_n)-F^i_k(x_1,\ldots,x_n)\Big|^pd\mu\\
           &&   + 2^{p-1}\sum_k\eta_k(y)\sum_{i=1}^n\int_{\Omega}\Big|F^i_k(x_1,\ldots,x_n)-W_k^i(x_1,\ldots,x_n)\Big|^pd\mu\\
           &<& 2^{p-1}\sum_k\eta_k(y)\frac{\epsilon^p}{2^p} + 2^{p-1}\sum_k\eta_k(y)\frac{\epsilon^p}{2^p}\\
           &=&\frac{\epsilon^p}{2} + \frac{\epsilon^p}{2}\\
           &=& \epsilon^p, 
    \end{eqnarray*}  
 where $\Omega =[a,t_1]\times \cdots \times [t_{n-1},b]$. Taking $p$-roots, we complete the proof.
\end{proof}

We are now in the position of proving the main result of this article. 

\begin{thm}\label{thm:approx_transf}
    Let $T: C^{k,\alpha}([a,b]) \longrightarrow C^{k,\alpha}([a,b])$ be an integral operator as above, and let $\mathbb K \subset C^{k,\alpha}([a,b])$ be compact. Then, for any $\epsilon>0$ there exists a generalized transformer architecture $\mathcal W$, and a discretization of $[a,b]$ such that 
    \begin{eqnarray*}
        \left(\sum_i \int_\Omega\Big|T^i(y)(x_1,\ldots, x_n) - \mathcal W^i_y(x_1,\ldots, x_n)\Big|^p \right)^{\frac{1}{p}} < \epsilon, 
    \end{eqnarray*}
    for every $y \in \mathbb K$, and every choice of $p\geq 1$, where $\Omega = [a,t_1]\times \cdots \times [t_{n-1},b]$ is determined by the discretization of $[a,b]$.
\end{thm}
\begin{proof}
        Without loss of generality, we set $[a,b] = [0,1]$. We let $\bar B_\rho$ denote a ball containing $\mathbb K$. Applying Lemma~\ref{lem:int_scheme} we find an integration scheme $\mathcal T$ such that 
        \begin{eqnarray*}
            \|T(y) - \mathcal T(y)\|_\infty < \frac{\epsilon}{2}, 
        \end{eqnarray*}
        with some discretization $0=t_1, \ldots , t_n = 1$. We can assume that the discretization is uniform, with subintervals of length $\frac{1}{n}$, since we can otherwise pass to a finer partition with this property, with $\mathcal T$ still satisfying the previous estimate.  
        Using Lemma~\ref{lem:equiv}, we find that $\mathcal T$ induces a continuous permutation equivariant map which we still denote by $\mathcal T$ by an abuse of notation. 
        We apply Lemma~\ref{lem:int_trans} as follows. Over the compact $\mathbb K$, there exists a generalized transformer $\mathcal W$ such that 
        \begin{equation*}
            \left (\sum_{i=1}^n \int_\Omega|\mathcal T^i_y(x_1,\ldots,x_n)- \mathcal W^i(y)(x_1,\ldots,x_n)|^pd\mu\right )^\frac{1}{p} < \frac{\epsilon}{2},
        \end{equation*}
         where we have set again for simplicity $\Omega = [0,t_1]\times \cdots \times [t_{n-1},1]$ to shorten notation, where the $t_i$ are the points of the aforementioned discretization. 

        Therefore, we have
        \begin{eqnarray*}
        \lefteqn{\sum_i \int_\Omega|T^i(y)(x_1,\ldots, x_n) - \mathcal W^i_y(x_1,\ldots, x_n)|^pd\mu}\\
        &\leq&  \sum_{i=1}^n 2^{p-1}\int_\Omega |T^i(y)(x_1,\ldots, x_n) - \mathcal T^i(y)(x_1,\ldots, x_n)|^pd\mu\\  
        && + \sum_{i=1}^n 2^{p-1}\int_\Omega|\mathcal T^i(y)(x_1,\ldots,x_n)-\mathcal W^i(y)(x_1,\ldots,x_n)|^pd\mu\\
        &<& \sum_{i=1}^n(2^{p-1}\frac{\epsilon^p}{2^pn^n} + 2^{p-1}\frac{\epsilon^p}{2^pn^n})\\
        &\leq& \frac{\epsilon^p}{2} + \frac{\epsilon^p}{2}\\
        &=& \epsilon^p
        \end{eqnarray*}
        for all $y$ in $\mathbb K$. Taking $p$-roots completes the proof.
\end{proof}

\begin{rmk}
    {\rm 
        We note that no special property in the proofs above has been used regarding the architecture of the transformer, except the fact that transformers can approximate equivariant sequence-to-sequence maps, in the norm \eqref{def:comp_pnorm}, with arbitrary accuracy. As a consequence, the approach of this article can be used to show that any universal approximator of equivariant sequence-to-sequence mappings can be used to approximate integral operators in H\"older space with arbitrary precision. Our main focus on transformers only stems from the fact that they are widely used in the study of operator learning problems. 
    }
\end{rmk}

\begin{rmk}
    {\rm 
        The approximation result of Theorem~\ref{thm:approx_transf} refers to a $p$-norm which is an average over the points where the tokens for the transformer can be sampled. While the error ``on average'' is small, for the learned operator, this does not mean that certain choices of values cannot incur in a significant spike in the error. This type of behavior has in fact been observed concretely in \cite{CST} where such error spikes were avoided by training the transformer model in a Sobolev space which functioned as regularization.
    }
\end{rmk}

While this is not the main scope of this article, we provide another universal approximation result for transformer architectures that implement a Leray-Schauder mapping, similar to \cite{zappala2024leray}. For this purpose, we define a Leray-Schauder transformer to consist of a Leray-Schauder mapping composed with either a traditional transformer architecture with positional encoding as in \cite{yun2019transformers,cordonnier2019relationship}, or the formulation of attention found in \cite{fang2022attention}. Then, we have the following result. 

\begin{thm}\label{thm:LR_transf}
    Let $T: X \longrightarrow Y$ be a possibly nonlinear operator between Banach spaces, and let $\mathbb K \subset X$ be compact. Then, fixed $\epsilon >0$ arbitrarily, there exists a Leray-Schauder transformer $\mathfrak T$ that approximates $T$ with precision $\epsilon$ on $\mathbb K$
    \begin{eqnarray*}
        \|T(x) - \mathfrak T(x)\|_Y < \epsilon,
    \end{eqnarray*}
    for all $x\in \mathbb K$. 
\end{thm}
\begin{proof}
    The proof consists of an application of results in \cite{EZ} and \cite{fang2022attention}, whose details will be omitted. From the proof of Theorem~2.1 in \cite{EZ}, we can approximate $T$ with precision $\frac{\epsilon}{2}$ over $\mathbb K$ with a Leray-Schauder mapping composed with a continuous function $f$. Since $f$ is defined over a compact (image of the Leray-Schauder mapping as in the proof of Theorem~2.1 in \cite{EZ}), we can use the results in \cite{fang2022attention,cordonnier2019relationship} to find a transformer architecture in the meaning mentioned above, to approximate $f$ uniformly with precision $\frac{\epsilon}{2}$. It can be seen that the resulting Leray-Schauder transformer architecture approximates $T$ with precision $\epsilon$ over $\mathbb K$.
\end{proof}

We notice that the result of Theorem~\ref{thm:LR_transf} gives a more general version of the results in \cite{shih2025transformers}, since their result is applicable to the case where $X$ has the approximation property, a restriction which is not required in the present work. Here the Leray-Schauder mappings play the role of the encoder used in \cite{shih2025transformers}.

\section{Universal approximation of operators by neural integral operators}\label{sec:NIE}

In this section we consider universal approximation property of operators in Banach spaces, where the approximator is constructed through a neural integral operator in a suitable generalized sense.

\begin{thm}\label{thm:integral}
    Let $X$ and $Y$ be Banach spaces, let $T: X \longrightarrow Y$ be an operator which is twice continuously Fr\'echet differentiable, and let $\mathbb K \subset X$ be compact. For any $\epsilon >0$, there exists an integral operator $U: X \longrightarrow Y$ such that 
        \begin{eqnarray}
            \|T(x) - U(x)\|_Y < \epsilon,
        \end{eqnarray}    
    for all $x\in \mathbb K$.
\end{thm}
\begin{proof}
    We use the Taylor expansion for nonlinear operators found in Kantorovich-Akilov \cite{kantorovich2016functional}, 
     where integration is intended as Riemann integral over a Banach space as in \cite{Gavurin}. See also \cites{gordon1991riemann,graves1927riemann}. Recall (\cites{kantorovich2016functional,Gavurin}) that for $R: [x_0, \bar x]\subset X \to  B(X,Y)$, where $B(X,Y)$ denotes the space of bounded linear operators from $X$ to $Y$, the integral of $R$ is defined (when it is well defined, e.g. when $R$ is continuous) as
    \begin{eqnarray*}
        \int_{x_0}^{x_0+\Delta x} R(x) dx = \int_0^1 R(x_0+t\Delta x) \Delta x \, dt = \lim_{n\to \infty} \sum\limits_{k=0}^{n-1} R(x_0+\tau_k\Delta x)\Delta x\,  (t_{k+1}-t_k), 
    \end{eqnarray*} 
    where $\tau_k\in [t_k,t_{k+1}]$ for all $k$. 
    For choice of $x_0 \in \mathbb K$ and for any $\bar x \in X$ such that $T$ is twice continuously Fr\'echet differentiable over the interval $[x_0, \bar x]$ (i.e. the set of vectors $t\bar x + (1-t)x_0$ with $t\in [0,1]$), we can write (Taylor's Theorem at degree $2$)
    \begin{eqnarray}
            T(\bar x) &=& T(x_0) + T'(x_0)(\bar x - x_0) + \int_{x_0}^{\bar x} T''(x)(\bar x - x , \cdot)dx, 
    \end{eqnarray}  
    where $\cdot$ is a placeholder for the entry of the bounded linear operator integrand. 
    Since $T''$ is continuous, for each point $x\in X$ we can find a $\delta'_x > 0$ such that $\|T''(z)-T''(x)\|_{B^2(X,Y)} < \epsilon$ whenever $\|z-x\|_X < \delta'_x$, where $B^2(X,Y)$ denotes the space of bounded bilinear maps from $X$ to $Y$ with the norm given by $\|A\| = \sup_{\|h\|_X\leq 1, \|k\|_X\leq 1} \|A(h,k)\|_Y$. Then, for each $x$ define $\delta_x = \min\{\delta'_x,1\}$. Since $\mathbb K$ is compact, we can cover it with finitely many balls $B(x_i,\delta_i)$, with $i=1,\ldots, n$, where $\delta_i$ is a shortened notation for $\delta_{x_i}$. We set $A(t)(\cdot) := (T''(x(t))-T''(x_i))((1-t)h,\cdot)$, with $x(t) = (1-t)h+x_i$, and where we again used $\cdot$ to indicate where the input of the linear operator $A(t)$ for each $t\in [0,1]$ is inserted. 
    In correspondence with the balls $B(x_i,\delta_i)$ we define
    \begin{eqnarray}
            T_i(\bar x) &=& T(x_i) + T'(x_i)(\bar x - x_i) + \int_{x_i}^{\bar x} T''(x_i)(\bar x - x , \cdot)dx. 
    \end{eqnarray}  
    Then, for any $x\in B(x_i,\delta_i)$, with $h = x - x_i$ and the parametrized interval $[x_i,x]\subset X$ given by $x(t) = (1-t)h+x_i$, one has
    \begin{eqnarray*}
        \Big{\|}\int_0^1 A(t)dt\Big{\|}_Y 
        &\leq& \int_0^1 \|A(t){\|}_Ydt\\
        &\leq& \int_0^1(1-t)\|h\|_X^2 \, \|T''(x(t))-T''(x_i){\|}_{B^2(X,Y)}dt\\
        &<& \epsilon. 
    \end{eqnarray*}
    Therefore, for each $x\in B(x_i,\delta_i)$ we have
    \begin{eqnarray*}
        \|T(x) - T_i(x)\|_Y
            &=& \Big{\|}T(x_i) + T'(x_i)(\bar x - x_i) + \int_{x_i}^x T''(z)(x - z, \cdot)dz\\
            && -  T(x_i) -  T'(x_i)(\bar x - x_i) - \int_{x_i}^x T''(x_i)(x - z, \cdot)dz\Big{\|}_Y\\
            &=& \Big{\|}\int_{x_i}^x(T''(z)-T''(x_i))(x-z,\cdot)dz\Big{\|}_Y\\
            &=& \Big{\|}\int_0^1(1-t)(T''((1-t)h+x_i))(h,h)dt\Big{\|}_Y\\
            &=& \Big{\|}\int_0^1 A(t)dt\Big{\|}_Y\\ 
            &<& \epsilon.
    \end{eqnarray*}
    The space $X$ is Hausdorff and paracompact, and we can therefore find a partition of unity $\{\eta_j\}_{j=1}^n\cup \{\eta_\infty\}$ subordinate to the open cover $\{B(x_i,\delta_i)\}_{i=1}^n\bigcup \{X-\mathbb K\}$, see \cite{Dugundji}, where we use the notation $\eta_\infty$ to denote the element of the partition with support contained in $X-\mathbb K$. We now define the operator $U : X \longrightarrow Y$ by patching together the operators $T_i : X \longrightarrow Y$. In fact, for any $x\in X$, we set
    $$
        U(x) = \sum_{i=1}^n \eta_i(x) T_i(x) + \eta_\infty(x)y_\infty,  
    $$
    where $y_\infty$ is arbitrarily chosen. 
    For any $x\in \mathbb K$, there exist $i_1, \ldots, i_d \neq \infty$ such that $x\in B(x_i,\delta_i)$ for $i = i_1, \ldots , i_d$ and $\sum_{i=i_1, \ldots, i_d} \eta_i(x) = 1$, while $\eta_i(x) = 0$ for all $i\neq i_1, \ldots, i_d$. Then, we have
    \begin{eqnarray*}
        \|T(x) - U(x)\| &=& \|\sum_i \eta_i(x)T(x) - \sum_i\eta_i(x)T_i(x)\|\\
                        &=& \|\sum_{i=i_1, \ldots, i_d} \eta_i(x)T(x) - \sum_{i=i_1, \ldots, i_d}\eta_i(x)T_i(x)\|\\
                        &\leq& \sum_{i=i_1, \ldots, i_d} \eta_i(x)\|T(x) - T_i(x)\|\\
                        &<& \sum_{i=i_1, \ldots, i_d} \eta_i(x) \epsilon\\
                        &=& \epsilon, 
    \end{eqnarray*}
    This completes the proof.
\end{proof}

\begin{thm}\label{thm:approx_neural_int}
    Let $T: X \longrightarrow Y$ and $\mathbb K$ be as in Theorem~\ref{thm:integral}. Then, for any choice of $\epsilon > 0$, there exists a Gavurin neural integral operator $\mathfrak T$ such that
    \begin{eqnarray*}
        \|T(x) - \mathfrak T(x)\|_Y < \epsilon,
    \end{eqnarray*}
    for any $x\in \mathbb K$. 
\end{thm}
\begin{proof}
         We apply Theorem~\ref{thm:integral} to approximate $T$ via an integral operator $U$, where 
\begin{multline*}
 U(y) = \sum_k \eta_k(y)T(y_k)+\sum_k \eta_k(y)T'(y_k)(y-y_k)\\  + \sum_k \eta_k(y)\int_0^1 T''(y_k)((1-t)(y-y_k),y-y_k)dt, 
\end{multline*} 
in such a way that $\|T(y) - U(y)\|_Y < \frac{\epsilon}{2}$, where $y_k$ are finitely many elements of $\mathbb K$ as in the proof of Theorem~\ref{thm:integral}. 

Let us consider the integral part of the operator $U$, namely 
\begin{eqnarray*}
    A(y) = \sum_k \eta_k(y)\int_0^1 T''(y_k)((1-t)(y-y_k),y-y_k)dt = \sum_k \eta_k(y)A_k(y),
\end{eqnarray*} where we have set $A_k(y) := \int_0^1 K_k(t,y)dt$, with $K_k(t,y) := T''(y_k)((1-t)(y-y_k),y-y_k)$. We have that $K_i$ is continuous with respect to $t$, and therefore uniformly continuous on $[0,1]$. We select $\delta>0$ such that $\|K_k(t_1,y) - K_k(t_2,y)\|_Y < \frac{\epsilon}{2}$ whenever $|t_1-t_2|<\delta$, and for each $y$ in the compact $\mathbb K$. Given this choice of $\delta$, we select a partition of points $0 = t_0, t_1, \ldots , t_{d-1}, t_d = 1$ such that $|t_i - t_{i+1}| < \delta$ for all $i = 0, 1, \ldots , d-1$. Using Theorem~2.1 in \cite{EZ}, we can find a neural network $\hat K_{ij}$, consisting of Leray-Schauder mappings and a deep neural network, such that $\|K_i(t_j)(y) - \hat K_{ij}(y)\|_Y < \frac{\epsilon}{2}$ for all $y\in \mathbb K$. Let us now consider the open covering $B_\delta(t_j)$ of open balls centered at $t_j$, with $j=0, \ldots, d$, and radius $\delta$. Let $\alpha_k$ denote a partition of unity subordinate to $\{B_\delta(t_j)\}$. One can directly verify that the operator $\hat K_i(t) := \sum_k \alpha_k(t)\hat K_{ik}$ satisfies the property that $\|K_i(t)(y)- \hat K_i(t)(y)\|_Y < \frac{\epsilon}{2}$ whenever $y$ is in $\mathbb K$. Since this construction can be preformed for each $i$, it follows that the integral part of the operator $U$ can be approximated by an integral operator with kernel consisting of Leray-Schauder mappings and (deep) neural networks. 

Another application of Theorem~2.1 in \cite{EZ} allows us to approximate $\sum_k \eta_k(y)T'(y_k)(y-y_k)$ with a deep neural network (and Leray-Schauder mappings) with precision $\frac{\epsilon}{2}$ over $\mathbb K$. Putting both approximations together, completes the proof.
\end{proof}

\section{Further perspectives}\label{sec:fp}

We conclude this article with some observations for further investigation. 

We would like to mention that the results of Theorem~\ref{thm:approx_transf} approximate integral operators between H\"older spaces, where the need of H\"older norms arises from the use of numerical quadrature to approximate integrals. Therefore, one might ask under what conditions transformers are universal approximators of integral operators between $L^p$ spaces. This question is yet to be answered. 

Theorem~\ref{thm:LR_transf} gives an approximation result for transformers in complete generality -- note that it refers to general Banach spaces, rather than some specific Banach space of functions. Once a norm is fixed, the Leray-Schauder mapping has an explicit form which is analytical, and therefore it seems that the new architecture is completely transparent. In general, however, it is useful to learn the Leray-Schauder mappings as well, to increase the applicability of the model. This is done for example in \cite{zappala2024leray}, where the resulting model is tested (in the $L^p$-space) on some integral equations and PDEs. The theoretical guarantees are also studied. It would be interesting to investigate extensions of these results in relation to the work in this article. 

Lastly, the results of Section~\ref{sec:NIE} show that (possibly) nonlinear mappings between general Banach spaces can be approximated through a generalized version of neural integral operator based on the Gavurin integral. While this is an interesting theoretical development, it is still unclear whether removing the use of the Gavurin integral would still allow for such a general approximation result.

\bibliographystyle{alpha}
\bibliography{refs}

\end{document}